\documentclass[11pt]{article}
\usepackage[margin=1in]{geometry}
\usepackage{amsmath,amssymb,amsthm}
\usepackage{graphicx}
\graphicspath{{./}}
\usepackage{booktabs}
\usepackage{natbib}
\usepackage{hyperref}
\usepackage{placeins}
\usepackage{xcolor}
\usepackage{mathptmx}

\newtheorem{proposition}{Proposition}

\title{When Actions Disappear:\\Adversarial Action Removal in Self-Play Reinforcement Learning}

\date{}
\author{
  Arahan Kujur\\
  Independent Researcher\\
  \texttt{kujurarahan@gmail.com}
}

\begin{document}
\maketitle

\begin{abstract}
We study adversarial action masking in self-play reinforcement learning: an attacker selectively removes legal actions from a victim's action set. Unlike observation or action perturbations, removal eliminates decision options before the agent acts. Across poker games scaling from 6 to 5,531 information states and two non-poker domains, learned masking causes substantially more damage than random masking and learned perturbation baselines. The attack persists across Q-learning, PPO, NFSP, neural NFSP, and DQN victims; transfers across agents; is amplified by self-play; and shows no recovery under extended masked training. Mechanistically, the adversary targets high-value decision points, captured by reach-weighted contingent action capacity (CAC$_w$) and a value-weighted refinement CAC$_v$. These results identify action availability as a distinct robustness surface in self-play RL.
\end{abstract}

\section{Introduction}

Multi-agent reinforcement learning (MARL) agents trained via self-play achieve strong performance in competitive domains~\citep{silver2018general,brown2019superhuman}, but their robustness to structural environment changes remains poorly understood. Prior work on adversarial attacks focuses on observation perturbations~\citep{huang2017adversarial,gleave2020adversarial} or reward manipulation~\citep{zhang2020adaptive}. We study a different, more severe attack surface: the \emph{action space} itself.

An adversary that selectively removes actions---disabling specific capabilities rather than adding noise---poses a qualitatively different threat from bounded perturbation. Such attacks arise naturally: hardware failures disable actuators, regulatory changes restrict strategies, API deprecations remove endpoints, and sandboxing constrains agent capabilities. We formalise this as a bi-level optimisation where the inner loop trains an RL agent under masked actions and the outer loop trains an adversary to choose which actions to remove.

Our key finding is that adversarial masking is \emph{dramatically more efficient} than random removal: across Leduc variants, learned removal is up to $4.8\times$ more damaging than random masking at comparable support. The mechanism operates through selective minimisation of reach-weighted contingent action capacity (CAC$_w$)---empirically verified by strong correlation between CAC$_w$ and victim reward across budget levels.

\paragraph{Contributions.}
\begin{itemize}
  \item We formalise adversarial action masking and show it is $4\times$ more damaging than learned perturbation (RARL-style) with equal training budget.
  \item We demonstrate scaling from 6 to \textbf{5,531} information states across five algorithms (QL, PPO, NFSP, neural NFSP, DQN), with the adversary's advantage increasing from $2.2\times$ to $4.8\times$ within the Leduc family as game complexity grows.
  \item We validate cross-domain in two non-poker environments (competitive gridworld, resource collection), confirming the phenomenon is not poker-specific.
  \item We connect the mechanism to CAC$_w$ ($r = 0.80$) and a refined CAC$_v$ ($r = 0.81$), and show victims do not recover under extended training.
\end{itemize}

\section{Related Work}

\paragraph{Adversarial attacks on RL.}
Adversarial attacks on RL most commonly perturb observations~\citep{huang2017adversarial,gleave2020adversarial,zhang2021robust,sun2022who} or poison rewards~\citep{zhang2020adaptive}. In multi-agent settings, recent work studies adversarial policies~\citep{gleave2020adversarial}, sparse or model-based attacks against cooperative MARL~\citep{lin2020marlrobust,hu2022sparse}, and collusive policy-level attackers~\citep{niu2026cama}. These attacks manipulate observations, policies, or rewards while leaving the legal action set intact. We instead study a structural attack on action availability itself.

\paragraph{Robust and action-robust RL.}
Robust MDPs~\citep{iyengar2005robust,nilim2005robust} and constrained MDPs~\citep{altman1999constrained} model uncertainty or constraints in the environment. Robust Adversarial RL (RARL)~\citep{pinto2017robust} and action-robust RL~\citep{tessler2019action} train against perturbations of the action actually executed. Those channels modify selected actions; our adversary removes actions from the legal set before selection. This difference matters: bounded perturbations preserve the agent's ability to choose, whereas removal can collapse a decision point to a singleton action set.

\paragraph{Robust MARL defenses.}
Recent robust MARL methods train agents against structured adversaries, including game-theoretic robust training for temporally coupled perturbations (GRAD)~\citep{liang2024grad} and fault-switching MARL defenses such as MARTA~\citep{mguni2025marta}. Fault-tolerant control has long studied actuator failures and reconfiguration~\citep{blanke2006fault}; our setting can be viewed as a learning-based analogue where failures are state-dependent and adversarially selected. These approaches are complementary: they improve robustness to perturbations, malfunctions, or adversarial agents within an intact joint action space. Our attack changes the feasible action set itself. A mask-aware version of such defenses is a natural direction, but existing guarantees do not directly cover state-dependent removal of legal actions.

\paragraph{Action sets and masking.}
Invalid action masking prevents agents from selecting impossible moves~\citep{huang2022closer}. Decision-theoretic planning has long studied structural leverage and action contingencies~\citep{boutilier1999decision}, and empowerment measures quantify an agent's available control over future outcomes~\citep{klyubin2005empowerment}. Our CAC view is related in spirit but is game-local and adversarial: it measures remaining multi-action decision capacity at reached information states. Our work reverses the usual masking motivation: the mask is not a safety or efficiency aid, but an attacker-chosen capability removal.

\paragraph{Self-play and imperfect-information games.}
Self-play can overfit, cycle, or exploit non-transitive structure~\citep{balduzzi2019open,lanctot2019openspiel}. Regret-minimization methods such as CFR~\citep{zinkevich2007regret} and Deep CFR~\citep{brown2019deep} provide strong baselines for imperfect-information games, while NFSP~\citep{heinrich2016deep} combines best-response learning with average-strategy tracking. Population methods such as PSRO~\citep{lanctot2017unified} maintain diversity. We show that these algorithmic stabilizers do not address a different failure mode: when the action set itself is structurally reduced, averaging and population diversity cannot restore eliminated strategic dimensions.

\section{Problem Formulation}

\paragraph{Game model.}
Consider a two-player zero-sum extensive-form game $\Gamma$ with information sets $\mathcal{I} = \mathcal{I}_0 \cup \mathcal{I}_1$. At information set $h \in \mathcal{I}_p$, player $p$ chooses from legal actions $A(h)$. Player~0 is the \emph{victim}; Player~1 is the opponent. Both learn via self-play.

\paragraph{Adversary definition.}
An \emph{action-removal adversary} is a mapping $\mathcal{M}: \mathcal{I}_0 \to 2^A$ that, at each victim information set $h$, selects a subset $\mathcal{M}(h) \subseteq A(h)$ of actions to \emph{retain}. The victim observes only $\mathcal{M}(h)$ and is unaware of the adversary. Formally:

\begin{itemize}
\item \textbf{Input}: information set $h$, legal actions $A(h)$, current player $p$.
\item \textbf{Output}: $\mathcal{M}(h) \subseteq A(h)$ with $|\mathcal{M}(h)| \geq 1$ (at least one action retained).
\item \textbf{Constraint}: $|\text{supp}(\mathcal{M})| = |\{h : |\mathcal{M}(h)| < |A(h)|\}| \leq k$ (budget).
\item \textbf{Objective}: minimise the victim's expected value $V_0(\pi^*_\mathcal{M})$.
\end{itemize}

\paragraph{Bi-level optimisation.}
The adversary and victim interact through a bi-level problem:
\begin{align}
\text{Inner:} \quad & \pi^*_\mathcal{M} = \arg\max_\pi \mathbb{E}\left[\sum_t r_t \mid \pi, \mathcal{M}\right] \label{eq:inner} \\
\text{Outer:} \quad & \mathcal{M}^* = \arg\min_{\mathcal{M} \in \mathcal{C}_k} V_0(\pi^*_\mathcal{M}) \label{eq:outer}
\end{align}
where $\mathcal{C}_k = \{\mathcal{M} : |\text{supp}(\mathcal{M})| \leq k\}$. The inner loop trains the victim under the adversary's mask via RL; the outer loop updates the adversary via REINFORCE with reward signal $-V_0$. In practice, we alternate: 500 episodes of inner training, then one adversary gradient step, for 20--25 outer iterations.

\paragraph{Adversary implementation.}
\emph{Tabular}: preference table $\theta(h, a)$ with $p_{\text{remove}}(a|h) = \text{softmax}(\theta(h,\cdot))$; the action with highest removal probability is removed. Updated via REINFORCE.
\emph{Neural}: MLP $f_\phi: \mathbb{R}^d \to \Delta^{|A|+1}$ mapping state features to a distribution over $|A|$ removal choices plus ``no removal.'' Trained via REINFORCE with mean-reward baseline.

\paragraph{Connection to CAC$_w$.}
Define \emph{reach-weighted contingent action capacity}:
\[
\text{CAC}_w(\mathcal{M}) = \sum_{h \in \mathcal{I}_0} \rho(h) \cdot \mathbf{1}[|\mathcal{M}(h)| > 1]
\]
where $\rho(h)$ is the reach probability of $h$ under current play. CAC$_w$ measures \emph{remaining} multi-action capacity: an unmasked decision state with at least two retained actions contributes positive capacity, while a state collapsed to a singleton contributes zero. Stronger attacks therefore drive CAC$_w$ downward by turning high-reach decision points into forced moves. The adversary's value-minimisation objective~\eqref{eq:outer} can be decomposed: masking state $h$ reduces $V_0$ by approximately $\rho(h) \cdot \delta(h)$, where $\delta(h) = |Q(h, a^*) - Q(h, a_{\text{forced}})|$ is the \emph{value gap} at $h$. This motivates a refined measure:
\[
\text{CAC}_v(\mathcal{M}) = \sum_{h \in \mathcal{I}_0} \rho(h) \cdot \delta(h) \cdot \mathbf{1}[|\mathcal{M}(h)| > 1]
\]
Empirically, CAC$_v$ correlates with victim reward at $r = 0.81$ vs.\ CAC$_w$'s $r = 0.77$ (Section~\ref{sec:efficiency}).

\paragraph{Deterministic Exploitation Attractor (DEA).}

\begin{proposition}[DEA Convergence]
Under CAC$_w = 0$ in self-play Q-learning with $\varepsilon$-greedy ($\varepsilon > 0$, $\alpha \in (0,1)$):
(i)~the victim's policy converges to the unique forced action at every information set;
(ii)~the opponent's Q-values converge to $Q^* = V^{\text{BR}(\sigma^*)}$;
(iii)~$(\sigma^*, \text{BR}(\sigma^*))$ is a stable fixed point.
\end{proposition}

\emph{Proof sketch.} Forced actions make the victim's policy Q-value-independent. The opponent faces a stationary MDP; Q-learning converges under standard conditions. The pair is stable: neither player can unilaterally deviate.

\begin{proposition}[Damage bound]
Let $\mathcal{M}$ mask a single information set $h^*$ by retaining a singleton action $a_f(h^*)$. The victim's value loss relative to unmasked play is bounded by:
\[
\Delta V_0 \geq \rho(h^*) \cdot [Q_0(h^*, a^*(h^*)) - Q_0(h^*, a_f(h^*))]
\]
where $a^*(h)=\arg\max_{a\in A(h)} Q_0(h,a)$ is the victim's best action, $a_f(h)$ is the action forced by the mask at $h$, and $\rho(h)$ is the reach probability. For a budget-$k$ adversary with support $S=\text{supp}(\mathcal{M})$, the additive approximation is:
\[
\Delta V_0(\mathcal{M}) \approx
\sum_{h \in S} \rho(h) \cdot [Q_0(h,a^*(h)) - Q_0(h,a_f(h))]
\]
and is upper-bounded by:
\[
\Delta V_0(\mathcal{M}) \leq
\sum_{h \in S} \rho(h) \cdot \max_{a_f \in A(h)\setminus\{a^*(h)\}}
[Q_0(h,a^*(h)) - Q_0(h,a_f)].
\]
Thus a greedy adversary selects states with high $\rho(h)\delta(h)$, where $\delta(h)=Q_0(h,a^*(h))-Q_0(h,a_f(h))$---exactly the states that maximise per-unit CAC$_v$ reduction.
\end{proposition}

\emph{Proof.} The bound follows from the definition of Q-values and linearity of expectation over the reach-weighted trajectory. At state $h^*$, forcing $a_f(h^*)$ instead of $a^*(h^*)$ loses $Q_0(h^*,a^*(h^*)) - Q_0(h^*,a_f(h^*))$ in expectation, weighted by $\rho(h^*)$. Under the simplifying assumption that masked states lie on independent branches, the resulting coverage objective suggests an approximate greedy structure.\footnote{This independence assumption is only approximate in extensive-form games: masking one state can alter reach probabilities and values downstream.}

\paragraph{Theoretical scope.}
Propositions 1--2 provide \emph{sufficient conditions} and \emph{local bounds}, not a complete characterization of the optimal adversary. The damage bound is per-state and additive, ignoring cross-state interactions (masking state $h$ may shift reach probabilities and Q-values at downstream states). A global characterization would require solving a combinatorial optimisation over $\binom{|\mathcal{I}_0|}{k}$ mask configurations---NP-hard in general via reduction to weighted max coverage. Under approximate independence, the value-gap decomposition motivates a greedy heuristic over high $\rho \cdot \delta$ states; we do not claim a general approximation guarantee for arbitrary extensive-form games. We view these bounds as explaining \emph{why} the adversary works rather than fully characterizing \emph{how well} it can work.

\section{Methods}

\paragraph{Self-play protocol.}
Both players share a single learning agent; the mask is applied to Player~0's actions at every decision point during \emph{both training and evaluation}. In the self-play regime, both sides adapt continuously. In the fixed-opponent regime, a snapshot of the agent's value function is frozen before masking begins; this static copy plays as Player~1 while only Player~0 continues learning under the mask.

\paragraph{Victim agents.}
We evaluate four algorithms spanning tabular and neural methods:
\begin{itemize}
\item \textbf{Tabular Q-Learning}: $\varepsilon$-greedy, $\varepsilon = 0.15$, $\alpha = 0.1$.
\item \textbf{Tabular PPO}: softmax policy, clipped surrogate, entropy bonus (Appendix~\ref{app:ppo}).
\item \textbf{Tabular NFSP}~\citep{heinrich2016deep}: best-response QL + average strategy, $\eta = 0.1$.
\item \textbf{DQN}: 2-layer MLP (64 hidden), experience replay (20k buffer), target network, Adam ($10^{-3}$). Full hyperparameters in Appendix~\ref{app:hyperparams}.
\end{itemize}

\paragraph{Masking strategies.}
\begin{table}[ht]
\centering
\caption{Masking strategies evaluated.}
\label{tab:strategies}
\begin{tabular}{ll}
\toprule
Strategy & Description \\
\midrule
None & Full action space (baseline) \\
Random($p$) & Remove each action independently with prob.\ $p$ \\
Fixed & Always remove a specific action (e.g., RAISE) \\
Adversarial (tabular) & Learned per-state removal via softmax preference table \\
Adversarial (neural) & MLP adversary trained with REINFORCE \\
Value heuristic & Remove at states with highest $|Q|$ \\
\bottomrule
\end{tabular}
\end{table}

\paragraph{Adversary architectures.}
Both tabular and neural adversaries are defined formally in Section~3. The neural adversary scales its architecture with game size: 2-layer (32 hidden) for Kuhn/Leduc, 3-layer (128-64) for Leduc-10/20.

\paragraph{Environments.}
\textbf{Kuhn Poker}: 3 cards, 2 actions, 6 P0 info states.
\textbf{Leduc Poker}: 6 cards (3 ranks $\times$ 2 suits), 3 actions, 2 rounds, $\sim$50 P0 info states.
\textbf{Leduc-5}: 10 cards (5 ranks), 389 P0 info states.
\textbf{Leduc-10}: 20 cards (10 ranks), 1,496 P0 info states.
\textbf{Leduc-20}: 40 cards (20 ranks), \textbf{5,531 P0 info states}.
\textbf{Competitive Gridworld}: 5$\times$5, 5 actions, 149 P0 states.
\textbf{Resource Collection}: 4$\times$4, 4 actions, 18,604 P0 states.
All implemented from scratch with \texttt{MaskedEnv} wrappers.

\section{Experiments}

We structure experiments around five questions: (1)~Does adversarial masking scale with game complexity? (2)~Is the vulnerability algorithm-invariant? (3)~Does it generalise beyond poker? (4)~How efficient is adversarial targeting? (5)~What mechanism drives the attack?

\subsection{Scaling with Game Complexity}
\label{sec:neural}

Our strongest result uses function approximation on both sides: a DQN victim and a neural MLP adversary in Leduc variants whose state spaces range from $\sim$50 to 5,531 victim information states.

\begin{figure}[ht]
\centering
\includegraphics[width=0.72\linewidth]{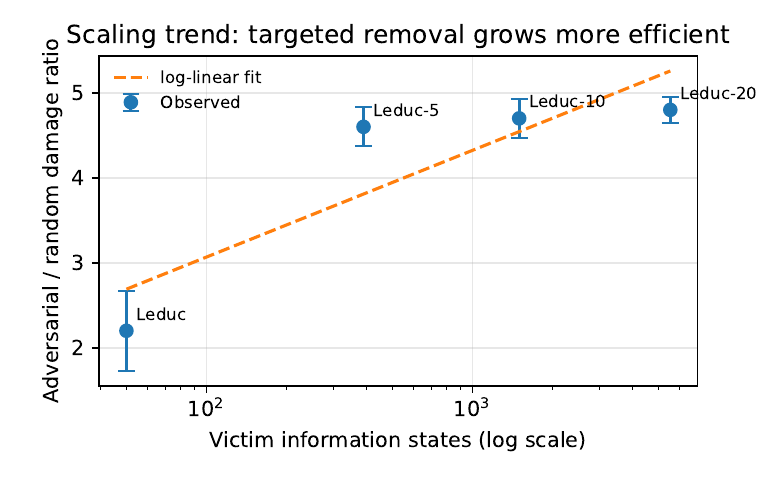}
\caption{Scaling trend. The adversarial/random damage ratio is plotted against victim information-set count on a log scale. Error bars show 95\% CIs from the DQN scale experiments.}
\label{fig:scaling}
\end{figure}

We implement Leduc-N, parameterised by rank count. \textbf{Leduc-5} (10 cards, 389 states), \textbf{Leduc-10} (20 cards, 1,496 states), and \textbf{Leduc-20} (40 cards, \textbf{5,531 states}) test whether the adversary's advantage persists as game complexity grows by $920\times$ from Kuhn. Figure~\ref{fig:scaling} and Table~\ref{tab:scaling} show a consistent empirical scaling trend.

\begin{table}[ht]
\centering
\caption{DQN victim across poker scales (5 seeds, 95\% CIs).}
\label{tab:scaling}
\begin{tabular}{lccccc}
\toprule
Game & $|\mathcal{I}_0|$ & None & Random & Neural Adv & Ratio \\
\midrule
Leduc & $\sim$50 & $-0.10$ & $-1.17$ & $-2.57 \pm 0.47$ & $2.2\times$ \\
Leduc-5 & 389 & $-0.15$ & $-0.60$ & $-2.74 \pm 0.23$ & $4.6\times$ \\
Leduc-10 & 1,496 & $-0.18$ & $-0.68$ & $-3.19 \pm 0.23$ & $4.7\times$ \\
Leduc-20 & \textbf{5,531} & $-0.09$ & $-0.63$ & $\mathbf{-3.00 \pm 0.15}$ & $\mathbf{4.8\times}$ \\
\bottomrule
\end{tabular}
\end{table}

At 5,531 states the neural adversary drives DQN to $-3.00$ per hand---$4.8\times$ worse than random---with the tightest CIs of any scale ($\pm 0.15$). A strict matched-$L_0$ control in Leduc samples exactly the same number of states as the learned mask ($k=64.8\pm4.3$): adversarial masking reaches $-2.32\pm0.36$ while matched random reaches $-1.03\pm0.24$ ($2.24\times$ gap). The advantage is therefore \emph{which} states are targeted, not support size.

\paragraph{Threat-model ablation.}
The main setting assumes environment-level control: the adversary can observe the victim information state and alter legal actions before selection. This is intentionally strong, matching settings where a simulator, platform, or operating layer controls capabilities. To test whether private-card access is essential, we restrict the adversary to public information only (betting history and public card, private rank hidden). In Leduc, public-info adversarial masking still reaches $-1.71 \pm 0.58$, substantially worse than random masking ($-0.98 \pm 0.08$), though weaker than private-info masking ($-2.30 \pm 0.39$). Thus strong private observability is not required: public structural information is sufficient for damaging targeted removal.

\paragraph{Neural NFSP on Leduc-5.}
To confirm the result holds with a stronger victim, we run \emph{Neural NFSP} (128-64 MLP average policy, DQN best-response, reservoir-sampled action buffer) on Leduc-5. Pre-attack: $-0.13 \pm 0.09$; post-adversarial: $\mathbf{-1.89 \pm 0.16}$; post-random: $-0.89 \pm 0.06$. The adversary is $2.3\times$ more damaging than random, confirming that even proper neural game-solving methods collapse under adversarial action removal.

\begin{figure}[ht]
\centering
\includegraphics[width=0.72\linewidth]{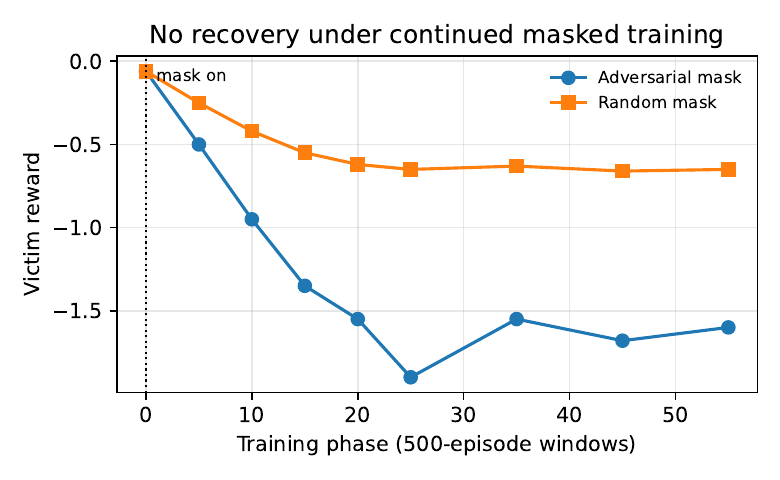}
\caption{Learning under a fixed mask. NFSP in Leduc-5 does not recover under continued masked training; adversarial removal remains substantially worse than random removal.}
\label{fig:learning}
\end{figure}

\paragraph{DQN masking mechanics.}
Masked actions are filtered from the legal action list \emph{before} $\varepsilon$-greedy selection: the DQN computes Q-values for all actions, then restricts argmax and exploration to the masked subset. The victim is unaware of the masking. Total DQN training: 20k pre-training episodes (no mask) + 20 outer $\times$ 500 inner = 10k episodes under the mask. Learning curves (Appendix~\ref{app:curves}) confirm convergence: reward stabilises by episode $\sim$8k under the mask.

\paragraph{OpenSpiel/CFR context.}
Kuhn and Leduc rules match the standard imperfect-information games used in OpenSpiel~\citep{lanctot2019openspiel}; we implement them from scratch only to expose the action-mask hook at every legal-action query and to scale Leduc-N by rank count. CFR and Deep CFR~\citep{zinkevich2007regret,brown2019deep} assume a fixed game tree. Action removal changes the game tree itself, so CFR remains a useful equilibrium reference but is not a direct defense unless retrained on the masked game. Our results should therefore be read as structural robustness tests of learning dynamics under game-tree modification.

\subsection{Tabular Validation Across Algorithms}
\label{sec:tabular}

\begin{table}[ht]
\centering
\caption{Tabular agents in Kuhn and Leduc under masking (5 seeds). Both raw reward ($r$) and normalised performance shown.}
\label{tab:tabular}
\begin{tabular}{lccccc}
\toprule
& \multicolumn{2}{c}{Raw Reward} & \multicolumn{2}{c}{Normalised} \\
\cmidrule(lr){2-3} \cmidrule(lr){4-5}
Setting & None & Advers. & None & Advers. \\
\midrule
Kuhn + QL & $-0.05$ & $-0.98$ & $.487$ & $.267$ \\
Kuhn + PPO & $-0.05$ & $-0.85$ & $.486$ & $.274$ \\
Kuhn + NFSP & $+0.12$ & $-0.45$ & $.530$ & $.388$ \\
Leduc + QL & $+0.05$ & $-3.50$ & $.502$ & $.367$ \\
Leduc + PPO & $-0.08$ & $-3.32$ & $.497$ & $.369$ \\
Leduc + NFSP & $-0.01$ & $-1.58$ & $.500$ & $.440$ \\
\bottomrule
\end{tabular}
\end{table}

All algorithms collapse under adversarial masking---in both Kuhn and Leduc. The vulnerability is \emph{algorithm-invariant}: Q-Learning, PPO, and NFSP all degrade to comparable levels. NFSP on Leduc drops from $-0.01$ to $-1.58$ (raw), confirming that NFSP's average-strategy component cannot compensate when the action space itself is reduced. Combined with the DQN scaling result (Table~\ref{tab:scaling}), we have \textbf{four algorithms across two games and two paradigms (tabular + neural) all exhibiting collapse}.

\subsection{Cross-Domain Validation: Competitive Gridworld}
\label{sec:crossdomain}

To verify the phenomenon is not a poker artifact, we test on a \emph{completely different} domain: a 5$\times$5 competitive gridworld where Player~0 (prey) navigates to a goal while Player~1 (predator) attempts to intercept. Both players have 5 actions (UP/DOWN/LEFT/RIGHT/STAY), the game is turn-based with perfect information, and 149 unique P0 states are observed.

\begin{table}[ht]
\centering
\caption{Cross-domain: Competitive Gridworld (5$\times$5, 5 actions, 149 P0 states, 5 seeds).}
\label{tab:gridworld}
\begin{tabular}{lcc}
\toprule
Masking & P0 Reward & $\Delta$ \\
\midrule
None & $+0.70 \pm 0.01$ & --- \\
Random ($p{=}0.3$) & $+0.04 \pm 0.18$ & $-0.66$ \\
Fixed (UP) & $+0.71 \pm 0.10$ & $+0.01$ \\
Adversarial & $\mathbf{-0.58 \pm 0.03}$ & $\mathbf{-1.27}$ \\
\bottomrule
\end{tabular}
\end{table}

The adversary causes $1.9\times$ more damage than random masking in a domain with no cards, no imperfect information, and a spatial rather than strategic structure. Fixed removal (always remove UP) has negligible effect because the prey reroutes; the adversary \emph{state-dependently} removes the action that matters most at each position.

\paragraph{Resource Collection (4$\times$4).}
A second non-poker domain: two agents compete to collect 4 resources on a grid, with reward = resources\_collected\_difference. 4 actions, 18,604 unique P0 states observed. Adversarial masking degrades P0 by $\Delta = -0.13$ vs.\ random $\Delta = -0.10$ ($1.4\times$). Fixed removal (UP) \emph{helps} P0 ($\Delta = +0.32$), showing that only state-dependent targeting produces consistent harm. The smaller gap ($1.4\times$ vs.\ $1.9\times$) reflects the symmetric, coordination-like structure where many states have similar strategic value---there is less exploitable structure to target.

Two non-poker domains with different structures (goal-seeking vs.\ resource competition, perfect information, spatial) both show adversarial $>$ random. The phenomenon is not poker-specific.

\subsection{Attack Efficiency and CAC$_w$ Correlation}
\label{sec:efficiency}

\paragraph{Budget sweep.}
Appendix~\ref{app:ablations} shows victim reward as a function of adversary budget $k$ (number of info states masked). Targeted removal degrades the victim as support grows; random masking requires broad coverage to approach the same damage.

\paragraph{CAC$_w$ correlation.}
To verify that the adversary operates through CAC$_w$ minimisation, we compute reach-weighted CAC$_w$ at each budget level and correlate with victim reward. Across all budget levels and both adversarial and random conditions (10 seeds each), the Pearson correlation between CAC$_w$ and reward is $r = 0.77$ ($p < 0.001$); for adversarial-only conditions, $r = 0.80$. The full budget table is in Appendix~\ref{app:ablations}.

\paragraph{CACv-greedy oracle.}
We also test a direct oracle heuristic that selects the top-$k$ Kuhn states by estimated $\rho(h)\delta(h)$ and removes the victim's estimated best action. At $k=3$, CACv-greedy reaches $-0.29 \pm 0.09$, stronger than random ($-0.19 \pm 0.06$) and the learned adversary under the same short training budget ($-0.22 \pm 0.02$). This supports CACv as a mechanistic proxy: when given the value-gap statistic directly, greedy targeting produces the expected harm.

\paragraph{Estimating CACv.}
We estimate $\rho(h)$ from on-policy visitation frequencies during evaluation rollouts and $\delta(h)$ from the victim's learned Q-table or network Q-values at the same checkpoint. Thus CACv is a diagnostic proxy, not an oracle ground truth: Q-estimation error can perturb rankings. Despite this estimation noise, CACv rankings are stable across seeds and align with oracle performance: the top CACv states in Kuhn repeatedly include the same high-impact information sets (\texttt{0pb}, \texttt{1}, \texttt{2pb}), and the CACv-greedy ablation remains stronger than random.

\subsection{Action Removal vs.\ Perturbation (RARL)}

\begin{table}[ht]
\centering
\caption{Action removal vs.\ perturbation: both with \emph{learned} adversary (Kuhn, 10 seeds, 95\% CIs). Both adversaries use the same REINFORCE training (20 outer $\times$ 500 inner). Removal is $4\times$ more damaging even when the perturbation adversary is given equal training budget.}
\label{tab:rarl}
\begin{tabular}{lcc}
\toprule
Attack Type & Raw Reward & Normalised \\
\midrule
None & $-0.02 \pm 0.07$ & $.494$ \\
Learned perturbation & $-0.24 \pm 0.09$ & $.440$ \\
Learned removal & $\mathbf{-1.01 \pm 0.09}$ & $\mathbf{.248}$ \\
\bottomrule
\end{tabular}
\end{table}

Both adversaries use identical training (REINFORCE, same budget, same iterations), ensuring a fair comparison. The learned perturbation adversary ($-0.24$) is stronger than fixed-$p$ RARL ($-0.09$, not shown), but learned removal ($-1.01$) remains $4\times$ more damaging. The mechanism is qualitatively different: perturbation adds noise to execution but preserves the agent's \emph{ability to choose}; removal eliminates decision points entirely, reducing CAC$_w$~\citep{pinto2017robust,tessler2019action}.

\paragraph{Perturbation protocol.}
In the learned perturbation baseline, the adversary observes the same information state and replaces the victim's chosen action with another legal action (or no-op) before execution. The victim is unaware of the replacement and receives the same training/evaluation schedule as in the removal experiment. This matches training budget and observability while preserving the legal action set.

\subsection{Additional Ablations}
The attack is amplified by self-play co-adaptation (self-play: $-0.975$, fixed opponent: $-0.841$), transfers across random seeds (transferred: $-1.032$ vs.\ per-agent retrained: $-0.872$), transfers from Q-learning to NFSP ($\Delta=-0.49$), and is not an artifact of parameter sharing: separate tabular Q-tables match shared-agent results, while separate DQN networks in Leduc still collapse to $-1.53 \pm 0.48$. Evaluation-only masking already degrades a normally trained Leduc victim ($-0.58$), while continued masked training and retraining from scratch on the fixed final learned mask converge to larger losses ($-2.65$ and $-2.71$), showing both immediate structural harm and long-run adaptation failure. Simple action-dropout helps modestly ($-1.82$ vs.\ $-2.45$), but a persistent random mask-ensemble defense does not help ($-2.64$). This negative result is informative: robustness to arbitrary stochastic unavailability is not enough; effective defenses likely need targeted protection or redundancy at high-CACv states. Full ablation tables are in Appendix~\ref{app:ablations}.

\section{Discussion}

\paragraph{Scaling trend.}
The scaling trajectory (Table~\ref{tab:scaling}) shows a consistent empirical trend. Excluding Kuhn (2-action ceiling effect), log-linear regression of adversary damage ratio on $\log_{10}(|\mathcal{I}_0|)$ yields $R^2 = 0.62$ (Pearson $r = 0.79$, $n = 4$ game sizes), with slope $1.56$. We report this as an \emph{empirical trend}, not a scaling law---four data points establish direction, not a precise functional form. The absolute damage gap (adversarial minus random reward) grows from $-1.40$ at 50 states to $-2.51$ at 1,496 states, and the adversary's variance \emph{decreases} with scale (CI $\pm 0.47 \to \pm 0.15$), suggesting more robust convergence in larger games where there is more targetable strategic heterogeneity. Neural NFSP ($2.3\times$ on Leduc-5) confirms this holds for properly neural game-solving methods.

\paragraph{CACw: necessary but not sufficient.}
The CAC$_w$--reward correlation ($r = 0.80$) is consistent with the hypothesis that reducing strategic flexibility drives performance degradation. However, CACw is a \emph{necessary} condition for collapse, not a complete characterization: at intermediate budgets, the adversary sometimes achieves comparable or greater damage with higher CACw than random masking (Table~\ref{tab:budget}, $k{=}3$: adversarial CACw $= 0.68$ vs.\ random $0.37$, yet adversarial reward $= -0.25$ vs.\ random $-0.21$). This occurs because CACw aggregates over all reachable states; the adversary targets the \emph{strategically pivotal} subset where action removal most disrupts equilibrium play, whereas random removal may eliminate more states but at less important ones. We test a refined measure, \emph{counterfactual-value-weighted capacity} $\text{CAC}_v = \sum_h \rho(h) \cdot |Q(h,a_0) - Q(h,a_1)| \cdot \mathbf{1}[|\mathcal{M}(h)| > 1]$, which weights each state by its Q-value gap (strategic importance). CAC$_v$ achieves Pearson $r = 0.81$ vs.\ CAC$_w$'s $r = 0.77$, confirming that incorporating value sensitivity improves the structural explanation. The remaining unexplained variance likely reflects higher-order interactions (e.g., how masking one state changes the value landscape at downstream states).

\paragraph{Implications for deployment.}
Real-world multi-agent systems where an adversary can disable specific capabilities (API endpoints, actuators, communication channels) are vulnerable to targeted collapse. Defences should maintain strategic flexibility at high-reach, high-CACv decision points. Uniform robustness methods such as random action dropout or random mask ensembles are insufficient in our experiments; the defensive object is not raw action count, but redundancy at the strategically pivotal states where removal produces large value gaps. Designing effective defenses likely requires identifying and preserving high-CACv states, rather than optimizing for uniform robustness to arbitrary action unavailability.

\section{Limitations}

Our largest game (Leduc-20) has $\sim$5,500 information states---nearly three orders of magnitude larger than Kuhn. The increasing adversary advantage with scale (Table~\ref{tab:scaling}) suggests the mechanism strengthens as games grow, but verification at full poker scale ($10^5$+ states) remains future work.

\paragraph{Discrete scope.}
Action removal is inherently a \emph{discrete} attack: it eliminates specific actions from a finite set. Our results---across poker, gridworld, and resource collection---all use small discrete action spaces ($|A| \leq 5$). In continuous-action domains (e.g., MuJoCo), the analogous attack would be \emph{region exclusion} (disabling subsets of the action manifold), which requires different formalisation and may exhibit qualitatively different dynamics. We do not claim generality beyond discrete action spaces.

\paragraph{Theory.}
The damage bound (Proposition~2) is local and additive, assuming independent masking effects across states. Cross-state interactions (masking one state shifts Q-values at others) are not captured. The scaling regression ($R^2 = 0.62$, 4 game sizes) establishes a trend but not a precise functional relationship.

\section{Conclusion}

We show that self-play RL agents---from tabular Q-Learning to neural NFSP and DQN, from 6-state Kuhn to 5,531-state Leduc-20, across poker and non-poker domains---are brittle to targeted action-space attacks. Within the Leduc family, the adversary's advantage over random masking increases with game complexity ($2.2\times$ at 50 states to $4.8\times$ at 5,531), and the vulnerability persists across five algorithms and seven environments. The mechanism operates through selective reduction of strategically important decision capacity (CAC$_w$/$_v$), victims do not recover under extended training, and the attack transfers across agents. These findings argue for robustness designs that preserve strategic flexibility at high-reach decision points rather than raw action count.

\section*{Reproducibility Statement}

All environments, agents, adversaries, and experiment scripts are implemented in the accompanying repository. Each reported result is generated by a standalone script under \texttt{experiments/}; random seeds are fixed in the scripts, and hyperparameters are listed in Appendix~\ref{app:hyperparams}. Figures are generated by \texttt{experiments/generate\_neurips\_figures.py} from completed experiment outputs. The largest run, Leduc-20, uses five seeds, 30k victim pre-training episodes, and 25 adversary outer iterations with 500 inner episodes each.

\bibliographystyle{plainnat}
\bibliography{references}

\appendix

\section{Tabular PPO Details}
\label{app:ppo}
Our tabular PPO maintains a softmax policy over a preference table $\theta(s,a)$. Action probabilities: $\pi(a|s) = \text{softmax}(\theta(s,\cdot))$. Updates use the clipped surrogate objective $L = \min(r_t A_t, \text{clip}(r_t, 1{-}\epsilon, 1{+}\epsilon) A_t)$ with $r_t = \pi_{\text{new}}/\pi_{\text{old}}$, advantage $A_t = R - b$ (running baseline), clip parameter $\epsilon = 0.2$, entropy bonus coefficient $0.01$, learning rate $0.01$.

\section{Experimental Protocols}
\label{app:protocols}

\paragraph{Self-play protocol.}
Both victim and opponent are the same agent; the mask is applied to Player~0's actions at every decision during both training and evaluation. In the fixed-opponent regime, a snapshot of the Q-table is frozen before masking; this static copy serves as Player~1 while only Player~0 continues learning.

\paragraph{Transfer experiment.}
The adversary is trained against one victim (seed~42) for 20 outer $\times$ 500 inner episodes. This mask is applied without modification to fresh victims (seeds 123--2048), each pre-trained 10k episodes then trained 10k under the transferred mask.

\paragraph{Budget enforcement.}
Budget $k$ is enforced via top-$k$ projection: the adversary trains over all states but at execution only the top-$k$ by confidence (max softmax prob.\ minus uniform baseline) have masks applied. This approximates a hard $L_0$ constraint on mask support.

\paragraph{Action-removal granularity.}
Unless otherwise stated, a masked state removes exactly one legal action; neural adversaries output one removal choice or a no-op. Masks never remove all actions: if removing the selected action would leave an empty set, the original legal set is retained. In two-action games, removing one action naturally forces a singleton; in three- or five-action games, most masks leave multiple alternatives.

\section{Hyperparameters}
\label{app:hyperparams}

\begin{table}[ht]
\centering
\caption{Q-Learning hyperparameters.}
\begin{tabular}{lcc}
\toprule
Parameter & Kuhn & Leduc \\
\midrule
Learning rate $\alpha$ & 0.1 & 0.1 \\
Exploration $\varepsilon$ & 0.15 & 0.15 \\
Discount $\gamma$ & 1.0 & 1.0 \\
Pre-training episodes & 10\,000 & 15\,000 \\
Post-mask episodes & 10\,000 & 15\,000 \\
\bottomrule
\end{tabular}
\end{table}

\begin{table}[ht]
\centering
\caption{DQN hyperparameters for Leduc-scale experiments.}
\begin{tabular}{lc}
\toprule
Parameter & Value \\
\midrule
Network (Leduc) & 37 $\to$ 64 $\to$ 64 $\to$ 3 \\
Network (Leduc-10/20) & $(2N{+}31) \to 128 \to 64 \to 3$ \\
Learning rate & $10^{-3}$ (Adam) \\
$\varepsilon$ schedule & $1.0 \to 0.05$ (decay $0.9999$) \\
Replay buffer & 20\,000--50\,000 \\
Batch size & 64--128 \\
Target update & every 500--1000 episodes \\
\bottomrule
\end{tabular}
\end{table}

\begin{table}[ht]
\centering
\caption{Neural NFSP hyperparameters (Leduc-5).}
\begin{tabular}{lc}
\toprule
Parameter & Value \\
\midrule
Best-response network & 49 $\to$ 128 $\to$ 64 $\to$ 3 \\
Average-policy network & 49 $\to$ 128 $\to$ 64 $\to$ 3 \\
BR learning rate & $10^{-3}$ \\
Average-policy learning rate & $10^{-3}$ \\
Reservoir buffer & 50\,000 action samples \\
Replay buffer & 20\,000 transitions \\
Anticipatory parameter $\eta$ & 0.1 \\
\bottomrule
\end{tabular}
\end{table}

\begin{table}[ht]
\centering
\caption{PPO hyperparameters.}
\begin{tabular}{lc}
\toprule
Parameter & Value \\
\midrule
Learning rate & 0.01 \\
Clip $\epsilon$ & 0.2 \\
Entropy bonus & 0.01 \\
Baseline & Running mean of returns \\
\bottomrule
\end{tabular}
\end{table}

\begin{table}[ht]
\centering
\caption{Adversary hyperparameters (tabular / neural).}
\begin{tabular}{lcc}
\toprule
Parameter & Tabular & Neural \\
\midrule
Architecture & Preference table & 37 $\to$ 32 $\to$ 4 (softmax) \\
Learning rate & 0.01 & $10^{-3}$ (Adam) \\
Outer iterations & 20--30 & 20--25 \\
Inner episodes & 500 & 500 \\
Update rule & REINFORCE & REINFORCE + baseline \\
\bottomrule
\end{tabular}
\end{table}

\section{Normalisation Details}
\label{app:normalisation}

Normalised performance uses $\text{norm}(r) = (r - r_{\min}) / (r_{\max} - r_{\min})$:

\begin{table}[ht]
\centering
\caption{Normalisation bounds (max single-hand loss/gain).}
\begin{tabular}{lcc}
\toprule
Game & $r_{\min}$ & $r_{\max}$ \\
\midrule
Kuhn Poker & $-2.0$ & $+2.0$ \\
Leduc Poker & $-13.0$ & $+13.0$ \\
Leduc-5 Poker & $-17.0$ & $+17.0$ \\
Leduc-10 Poker & $-21.0$ & $+21.0$ \\
Leduc-20 Poker & $-41.0$ & $+41.0$ \\
Comp.\ Gridworld & $-1.0$ & $+1.0$ \\
Resource Collection & $-4.0$ & $+4.0$ \\
\bottomrule
\end{tabular}
\end{table}

Normalisation aids cross-game comparison but obscures absolute exploitability: a normalised $.37$ in Leduc ($\Delta r = -3.5$) represents a larger absolute loss than $.27$ in Kuhn ($\Delta r = -0.9$).

\section{Additional Ablations}
\label{app:ablations}

\begin{table}[ht]
\centering
\caption{Mask timing controls in Leduc Q-learning.}
\begin{tabular}{lc}
\toprule
Training/evaluation protocol & Victim reward \\
\midrule
No mask & $+0.05 \pm 0.03$ \\
Evaluation-only mask & $-0.58 \pm 0.18$ \\
Continued training under mask & $-2.65 \pm 0.30$ \\
Fixed-mask retraining from scratch & $-2.71 \pm 0.26$ \\
\bottomrule
\end{tabular}
\end{table}

\begin{table}[ht]
\centering
\caption{Threat-model ablation in Leduc Q-learning: private vs.\ public information adversary.}
\begin{tabular}{lc}
\toprule
Adversary information & Victim reward \\
\midrule
None & $+0.05 \pm 0.03$ \\
Random mask & $-0.98 \pm 0.08$ \\
Public info only & $-1.71 \pm 0.58$ \\
Private info & $-2.30 \pm 0.39$ \\
\bottomrule
\end{tabular}
\end{table}

\begin{table}[ht]
\centering
\caption{CACv-greedy oracle at budget $k=3$ in Kuhn.}
\begin{tabular}{lc}
\toprule
Masking strategy & Victim reward \\
\midrule
None & $-0.02 \pm 0.08$ \\
Random & $-0.19 \pm 0.06$ \\
Learned adversary & $-0.22 \pm 0.02$ \\
CACv-greedy oracle & $\mathbf{-0.29 \pm 0.09}$ \\
\bottomrule
\end{tabular}
\end{table}

\begin{table}[ht]
\centering
\caption{Strict matched-$L_0$ control in Leduc Q-learning. Matched random samples exactly the same number of information states as the learned adversarial mask.}
\begin{tabular}{lcc}
\toprule
Masking strategy & Effective $k$ & Victim reward \\
\midrule
Adversarial & $64.8 \pm 4.3$ & $-2.32 \pm 0.36$ \\
Matched random & $64.8 \pm 4.3$ & $-1.03 \pm 0.24$ \\
\bottomrule
\end{tabular}
\end{table}

\begin{table}[ht]
\centering
\caption{Precise effective $L_0$ support diagnostics in Leduc Q-learning.}
\begin{tabular}{lccc}
\toprule
Mask & Masked states & Seen states & Decision mask rate \\
\midrule
Random & 69.0 & 74.0 & 0.761 \\
Fixed raise & 61.4 & 74.0 & 0.941 \\
Private adversary & 47.6 & 70.8 & 0.656 \\
Public adversary & 55.0 & 70.6 & 0.794 \\
\bottomrule
\end{tabular}
\end{table}

\begin{table}[ht]
\centering
\caption{Neural and defense ablations.}
\begin{tabular}{lcc}
\toprule
Ablation & Setting & Victim reward \\
\midrule
Separate DQN networks & Leduc + adversary & $-1.53 \pm 0.48$ \\
Defense & Standard training & $-2.45 \pm 0.60$ \\
Defense & Stochastic dropout & $-1.82 \pm 0.31$ \\
Defense & Random mask ensemble & $-2.64 \pm 0.54$ \\
\bottomrule
\end{tabular}
\end{table}

\begin{table}[ht]
\centering
\caption{Representative compute and sample costs. Wall-clock times are CPU runs on a desktop workstation.}
\begin{tabular}{lccc}
\toprule
Experiment & Seeds & Episodes per seed & Wall-clock \\
\midrule
Leduc DQN scale & 5 & 20k pre + 10k attack & $\sim$20 min \\
Leduc-20 DQN scale & 5 & 30k pre + 12.5k attack & $\sim$17 min \\
Reviewer ablations & 5--10 & 6k--24k & $\sim$5 min \\
Figure generation & n/a & none & $<$1 min \\
\bottomrule
\end{tabular}
\end{table}

\begin{table}[ht]
\centering
\caption{Budget sweep: victim reward (raw) under adversarial and random masking (Kuhn, 10 seeds, 95\% CIs).}
\label{tab:budget}
\begin{tabular}{lcc}
\toprule
$k$ & Adversarial & Random \\
\midrule
0 & $-0.02 \pm 0.07$ & $-0.02 \pm 0.07$ \\
1 & $-0.21 \pm 0.07$ & $-0.02 \pm 0.07$ \\
2 & $-0.24 \pm 0.07$ & $-0.10 \pm 0.06$ \\
3 & $-0.25 \pm 0.04$ & $-0.21 \pm 0.03$ \\
4 & $-0.28 \pm 0.04$ & $-0.41 \pm 0.04$ \\
5 & $-0.72 \pm 0.21$ & $-0.66 \pm 0.01$ \\
6 & $-0.98 \pm 0.07$ & $-0.92 \pm 0.00$ \\
\bottomrule
\end{tabular}
\end{table}

\begin{table}[ht]
\centering
\caption{Secondary ablations in Kuhn Q-learning.}
\begin{tabular}{lcc}
\toprule
Ablation & Setting & Victim reward \\
\midrule
Self-play & Co-adaptive & $-0.975 \pm 0.09$ \\
Self-play & Fixed opponent & $-0.841 \pm 0.36$ \\
Transfer & Trained on seed 42 & $-1.032 \pm 0.009$ \\
Transfer & Retrained per-agent & $-0.872 \pm 0.134$ \\
Parameter sharing & Shared Q-table & $-0.948 \pm 0.07$ \\
Parameter sharing & Separate Q-tables & $-0.948 \pm 0.07$ \\
\bottomrule
\end{tabular}
\end{table}

\begin{table}[ht]
\centering
\caption{Learned Kuhn mask at full budget.}
\begin{tabular}{llcc}
\toprule
State & Meaning & Removed & Confidence \\
\midrule
\texttt{0} & J, root & BET & $0.99$ \\
\texttt{0pb} & J, facing bet & PASS & $0.99$ \\
\texttt{1} & Q, root & PASS & $0.95$ \\
\texttt{2} & K, root & BET & $0.98$ \\
\texttt{2pb} & K, facing bet & BET & $0.98$ \\
\bottomrule
\end{tabular}
\end{table}

\section{Learning Curves Under Attack}
\label{app:curves}

To confirm that reported results are \emph{asymptotic} (converged) rather than transient, we report reward in 500-episode windows during post-attack training.

\paragraph{Kuhn + QL (seed 42).}
Pre-attack final window: $-0.16$. Post-attack trajectory: reward drops to $\sim\!-0.3$ within 2k episodes and oscillates around $-0.2$ to $-0.4$ for the remaining 13k episodes. No recovery trend.

\paragraph{Leduc + QL (seed 42).}
Pre-attack: $+0.13$. Post-attack: gradual degradation from $0$ to $-1.0$ over 10k episodes, with oscillations. Final window: $-0.40$. The attack takes $\sim$5k episodes to fully manifest as the adversary's policy-gradient updates progressively sharpen.

\paragraph{Leduc + NFSP (seed 42).}
Pre-attack: $-0.30$. Post-attack: initial spike to $+1.0$ (NFSP's average strategy briefly exploits the adversary's early exploration), followed by monotonic decline to $-0.51$ by episode 15k. The adversary overcomes NFSP's initial resistance.

\paragraph{Leduc-5 + NFSP (5 seeds).}
Pre-attack: $-0.06$. Under adversarial masking with 25 outer iterations: progressive degradation from $-0.5$ (outer 5) to $-1.9$ (outer 25). Continued training for 15k additional episodes under the converged mask: reward oscillates between $-1.2$ and $-2.1$ with no recovery trend. The adversary overcomes NFSP's average-strategy buffer at this scale.

These curves confirm that (a)~the attack converges and is not a transient phenomenon, (b)~victims do not recover even with $3\times$ extended training under the mask, and (c)~the no-recovery property holds across algorithms and game sizes.

\end{document}